\documentclass{article}

\PassOptionsToPackage{round}{natbib}
\usepackage[dblblindworkshop, preprint]{neurips_2026}

\usepackage[utf8]{inputenc} 
\usepackage[T1]{fontenc}    
\usepackage[x11names]{xcolor}         
\usepackage[pdfborder={0 0 0}, colorlinks=true, citecolor=Blue4, linkcolor=Blue4, urlcolor=Magenta4]{hyperref}       
\usepackage{url}            
\usepackage{booktabs}       
\usepackage{amsfonts}       
\usepackage{nicefrac}       
\usepackage{microtype}      

\usepackage{xspace}
\usepackage{amssymb, amsmath}
\usepackage{amsthm, thmtools}
\usepackage{mathtools}
\usepackage{dsfont}
\usepackage{enumitem}

\newcommand*{\argmin}{\operatornamewithlimits{argmin}}

\DeclarePairedDelimiter{\norm}{\lVert}{\rVert}
\DeclarePairedDelimiter{\abs}{\lvert}{\rvert}
\DeclarePairedDelimiter{\inner}{\langle}{\rangle}
\DeclarePairedDelimiter{\bracks}{\lbrack}{\rbrack}        
\DeclarePairedDelimiter{\paren}{\lparen}{\rparen}                       
\DeclarePairedDelimiter{\braces}{\lbrace}{\rbrace}          
\let\pairing\inner

\newcommand*{\placeholder}{\makebox[1.1ex]{$\cdot$}}
\newcommand*{\pch}{\placeholder}

\newcommand*{\field}[1]{\mathbb{\MakeUppercase{#1}}}		
\newcommand*{\set}[1]{{\mathcal{\MakeUppercase{#1}}}}			
\newcommand*{\collection}[1]{{\mathfrak{\MakeUppercase{#1}}}} 

\newcommand*{\seqidx}{n}                                        
\newcommand*{\seqbracks}{\paren}
\newcommand*{\seq}[2][\seqidx]{\seqbracks{#2_{#1}}_{#1=1}^\infty}              
\newcommand*{\range}[1]{\braces{1, \dotsc, #1}}                             

\newcommand*{\operator}[1]{{\MakeUppercase{#1}}} 

\newcommand*{\R}{\field{R}} 
\newcommand*{\N}{\field{N}} 

\newcommand*{\multi}[2][t]{#2_{1:#1}}
\renewcommand{\vec}[1]{{\boldsymbol{\mathbf{#1}}}}
\newcommand*{\mat}[1]{\vec{\MakeUppercase{#1}}}
\newcommand*{\eye}{\mat{I}}							
\newcommand*{\transpose}{\mathsf{T}}
\newcommand*{\tr}{{\operatorname{tr}}}						
\newcommand*{\trace}{\tr}

\newcommand*{\eigval}{\lambda}

\newcommand*{\eigvec}{v}

{}

\newcommand*{\observation}{y} 					
\newcommand*{\obsnoise}{\varepsilon}

\newcommand*{\dimension}{d}
\newcommand*{\anydim}{d}
\newcommand*{\locdim}{\dimension}

\newcommand*{\expectation}{\mathbb{E}}
\newcommand*{\E}{\expectation}
\newcommand*{\Ex}[1]{\E\bracks*{#1}}

\newcommand*{\prob}[1]{\mathbb{P}\left[ #1 \right]}

\newcommand*{\rv}{\eta}								

\newcommand*{\as}{\quad {\mathrm{(a.s.)}}}			

\newcommand*{\filtration}{\collection{F}}
\newcommand*{\anyevent}{\set{A}}
\newcommand*{\tstop}{\tau}

\newcommand*{\Hspace}{\set{H}} 						

\newcommand*{\vnoise}{\xi}
\newcommand*{\vncov}{\operator{\Sigma}}

\newcommand*{\outH}{\set{V}}
\newcommand*{\inH}{\set{U}}

\newcommand*{\adj}{*}                       

\newcommand*{\dual}{\adj}                   

\newcommand*{\lopspace}{\set{L}}
\newcommand*{\fspace}{\set{F}}          

\newcommand*{\idop}{\operator{I}}       
\newcommand*{\proj}{\operator{P}}
\newcommand*{\onb}{e}	

\newcommand*{\basisop}{\operator{\onb}}
\newcommand*{\Schatten}{\set{S}}
\newcommand*{\psnorm}{{p}}
\newcommand{\opnorm}{{\mathrm{op}}}

\newcommand*{\infspace}{\set{U}}
\newcommand*{\outfspace}{\set{V}}
\newcommand*{\discretization}{\operator{D}}

\newcommand*{\statenoise}{\epsilon}
\newcommand*{\state}{x}
\newcommand*{\dyn}{\text{dyn}}

\newcommand*{\regfactor}{\lambda}
\newcommand*{\regop}{\operator{U}}
\newcommand*{\regext}{\operator{W}}           
\newcommand*{\regmat}{\mat{V}}
\newcommand*{\obsop}{\mathrm{M}}              
\newcommand*{\obsops}[1][t]{\multi[#1]{\obsop}}
\newcommand*{\loss}{\ell}                     
\newcommand*{\Loss}{L}                        
\newcommand*{\obsspace}{\set{Y}}
\newcommand*{\mdlop}{\operator{F}}            
\newcommand*{\mdlspace}{\set{H}}              
\newcommand*{\targetop}{{\mdlop_\star}}       
\newcommand*{\learntop}{{\widehat{\mdlop}}}   

\newcommand*{\sumcov}{\operator{V}}
\newcommand*{\regsumop}{\operator{C}}

\newcommand*{\testop}{\operator{\Phi}}        

\newcommand*{\invec}{u}
\newcommand*{\outvec}{v}

\newcommand*{\outfun}{g}

\newcommand*{\infun}{u}
\newcommand*{\martingale}{\rho}
\newcommand*{\sumrv}{\operator{S}}

\newcommand*{\opkernel}{\operator{K}}

\newcommand*{\parameter}{\theta}
\newcommand*{\parameters}{\vec{\parameter}}

\newcommand*{\paramspace}{\Theta} 
\newcommand*{\paramcov}{\mat{\Sigma}}
\newcommand*{\paramrkhs}{\Hspace_\paramspace}
\newcommand*{\paramkernel}{\opkernel_\paramspace}

\newcommand*{\iteridx}{t}

\newcommand*{\anyscalar}{a}
\newcommand*{\anyvec}{\vec{v}}
\newcommand*{\anyvecelem}{v}
\newcommand*{\anyfunction}{h}

\newcommand*{\anyoperator}{\operator{A}}
\newcommand*{\anotheroperator}{\operator{B}}

\newcommand*{\half}{{\nicefrac{1}{2}}}
\newcommand*{\fhalf}{\frac{1}{2}}

\def\mmiddle#1{\mathrel{}\middle#1\mathrel{}} 

\newcommand*{\given}{\mmiddle{|}}

{}
\let\emptyset\varnothing{}

\newcommand*{\iid}{i.i.d.\xspace}

\declaretheorem[numberwithin=section]{theorem}
\declaretheorem[numberlike=theorem]{lemma}
\declaretheorem[numberlike=theorem]{corollary}

\declaretheorem{definition}

\newlist{assumptions}{enumerate}{1}
\setlist[assumptions]{label=\textnormal{\textbf{A\thetheorem.\arabic{*}}.}, ref=A\thetheorem.\arabic*}

\workshoptitle{AI for Stochastic Dynamics: From Theoretical Foundations to Scientific Applications}

\title{Sequential operator learning under dependent data}

\author{%
  Rafael Oliveira\\
  CSIRO Technology\\
  Sydney, Australia\\
  \texttt{rafael.dossantosdeoliveira@csiro.au}
}

\begin{document}

\maketitle

\begin{abstract}
Learning operators from sequentially collected data arises in adaptive experimental design, Bayesian optimization, and dynamical-system modelling, where observations may be dependent, and future inputs or sensing operators may depend on preceding data. We derive time-uniform self-normalized concentration bounds for stochastic processes in Hilbert spaces with vector-valued noise. We use these bounds to obtain regression-error guarantees for linear operators, including targets outside the Hilbert estimation space, and for nonlinear parametric operators trained with strongly convex losses and regularizers. Our results allow possibly infinite-dimensional inputs and outputs without independence or mixing assumptions, providing a major step towards convergence guarantees for adaptive operator learning and learning from stochastic dynamical data.
\end{abstract}

\section{Introduction}

Operator learning provides a framework for learning mappings between function spaces, with applications throughout scientific machine learning and dynamical-system modelling \citep{Kovachki2023,Boulle2024}. Prominent nonlinear architectures include neural operators \citep{Kovachki2023} and DeepONets \citep{Lu2021}. Sequential dependence arises naturally in weather forecasting, where models such as FourCastNet learn evolution operators from temporally ordered spatial fields \citep{Kurth2023fourcastnet}, and long-horizon climate modelling \citep{Wang2026multiscale}. More generally, scientific observations may form a dependent trajectory, while future inputs, forcing conditions, or sensing operators may be selected using the evolving model. These settings fall outside the fixed-design or independent-data assumptions underlying many existing learning guarantees.

Adaptive data collection for operator learning has recently been studied through optimal experimental design \citep{Xu2026}, active learning \citep{Li2024mrafno,Musekamp2025,Subedi2025onthebenefits}, and Bayesian optimization \citep{Guilhoto2024,Oliveira2025nots}. These settings may also involve partially observed function-valued operator outputs (e.g., through finite grids or sensor networks) together with structured noise. Hence, available theoretical guarantees remain limited for such general adaptive settings.

\paragraph{Contribution.}
Building on classical self-normalized concentration \citep{Pena2009article,Abbasi-Yadkori2011} and recent vector-valued extensions \citep{Chowdhury2021,Chugg2025variational,Martinez-Taboada2026}, we derive a time-uniform concentration inequality for general Hilbert-valued stochastic processes and use it to establish regression-error guarantees for linear and nonlinear operators learned from sequentially dependent data. Our linear result accommodates targets more general than those represented by the Hilbert estimation space, while our nonlinear result covers parametric operator models trained with a general class of loss functions and regularizers. The latter builds on an operator-valued extension of a model-based reproducing kernel Hilbert space (RKHS) construction \citep{Oliveira2026kernel}. Together, these results provide a common concentration framework for analyzing operator learning from dependent trajectories and adaptively collected observations.

\paragraph{Related work.}
Operator-learning theory addresses several complementary notions of approximation and estimation error. Approximation results establish the expressivity of Fourier neural operators (FNOs) and DeepONets and provide architecture-dependent error estimates \citep{Kovachki2021,Lanthaler2022error}. More recent statistical analyses control generalization for nonlinear operators, including neural operators, but typically assume independently sampled training data \citep{Lee2024,Reinhardt2024}. For linear operators, vector-valued RKHS theory provides a classical regression framework \citep{Micchelli2003,Carmeli2010}. \citet{Chowdhury2020} derive adaptive confidence bounds for a particular operator-valued regresson problem arising from conditional distributions, although their vector-valued concentration argument has a gap. \citet{Chowdhury2022} later obtain a valid but looser bound under stronger Hilbert-Schmidt assumptions. Related linear guarantees include \citet{Subedi2025onthebenefits} and \citet{Wang2026bovvrkhs}, the latter requiring trace-class operator-valued kernels. Our analysis instead allows for non-trace-class kernels and extends to the case of nonlinear regression problems.

Dependent observations have also been studied in dynamical-system learning: \citet{Foster2020learning} derive single-trajectory guarantees for structured finite-dimensional nonlinear systems, while \citet{Hou2023sparse} analyze RKHS-based operator estimation under \iid or mixing assumptions. Our results instead provide time-uniform guarantees in possibly infinite-dimensional spaces under general predictable data collection, without independence or mixing assumptions. 

\section{Background and problem formulation}
\label{sec:background}

\paragraph{Notation.} We work with separable real Hilbert and Banach spaces. For a Banach space $\fspace$, its dual is denoted by $\fspace^*$. We use $\inner{\cdot,\cdot}$ for both inner products and dual pairings, with the meaning determined by the spaces of the arguments. In particular, $\inner{f,\varphi}=f(\varphi)$ for $f\in\fspace^*$ and $\varphi\in\fspace$. Norms are denoted by $\norm{\cdot}$, with a space subscript added when needed for clarity. We write $\lopspace(\inH,\outH)$ for the bounded linear operators from $\inH$ to $\outH$ and $\anyoperator^\adj$ for the Banach or Hilbert adjoint, as appropriate. For a positive-semidefinite operator $\anotheroperator$ on a Hilbert space, let $\norm{\anyfunction}_\anotheroperator^2:=\inner{\anyfunction,\anotheroperator\anyfunction}$, and write $\anyoperator\preceq\anotheroperator$ when $\anotheroperator-\anyoperator$ is positive semidefinite. The Schatten classes are denoted by $\Schatten_p(\inH,\outH)$, with $p=1$ and $p=2$ corresponding to trace-class and Hilbert--Schmidt operators. For trace-class $\anyoperator$, $\det(\idop+\anyoperator)$ denotes the Fredholm determinant. More detailed definitions are presented in \autoref{app:bg}.

\paragraph{Sequential operator learning.}
Let $\mdlop_\star:\infspace\to\outfspace$ be an unknown, possibly nonlinear operator between Hilbert spaces which we want to approximate. At iteration $\iteridx\in\N$, we observe:
\begin{equation*}
    \observation_\iteridx
    =
    \discretization_\iteridx
    \mdlop_\star(\infun_\iteridx)
    +
    \vnoise_\iteridx
    \in\obsspace_\iteridx,
\end{equation*}
where $\discretization_\iteridx\in\lopspace(\outfspace,\obsspace_\iteridx)$ may represent partial observation or discretization mapping to a Hilbert space $\obsspace_\iteridx$. We assume that $\infun_\iteridx$ and $\discretization_\iteridx$ are predictable with respect to a filtration $\braces{\filtration_\iteridx}_{\iteridx\geq0}$ representing the information available to the algorithm (the observed data and any algorithmic randomness revealed up to each round), while $\vnoise_\iteridx$ is $\filtration_\iteridx$-measurable and conditionally sub-Gaussian noise, given $\filtration_{\iteridx-1}$. Thus, both the inputs and observation operators may depend arbitrarily on the preceding data.

This formulation includes stochastic dynamics. For example, if $\state_{\iteridx+1}=\mdlop_\dyn(\state_\iteridx)+\statenoise_\iteridx$ and $\observation_\iteridx=\discretization_\iteridx\state_\iteridx$, then $\mdlop_\star$ may represent the predictable map from the current observation or an estimated latent state to the next function-valued state. The inputs need only be predictable representations of the available history, so the observed process need not be Markovian or fully observed. The same formulation covers deliberate adaptivity, where future observations are chosen based on available data.

\section{Main result}
\label{sec:main}
Our main technical result extends self-normalized concentration to general Hilbert-valued increments. Supporting definitions are provided in \autoref{app:bg}, and our proofs are located in \autoref{app:main-proofs}.

\begin{theorem}[Self-normalized vector-valued concentration]
\label{thr:sn-sequence}
Let $\Hspace$ be a separable real Hilbert space and $\braces{\filtration_\iteridx}_{\iteridx=0}^\infty$ a filtration. Let $\braces{\rv_\iteridx}_{\iteridx=1}^\infty$ be an adapted $\Hspace$-valued sequence such that $\rv_\iteridx$ is conditionally $\vncov_\iteridx$-sub-Gaussian, where $\vncov_\iteridx$ is a predictable positive-semidefinite trace-class operator. Then, for any boundedly invertible positive-definite operator $\regop$ on $\Hspace$ and any $\delta\in(0,1]$,
\begin{equation}
    \prob{
        \forall\iteridx\in\N,\quad
        \norm*{
            \sum_{i=1}^\iteridx\rv_i
        }_{(\regop+\sumcov_\iteridx)^{-1}}^2
        \leq
        2\log\paren*{
            \frac{
                \det\paren{
                    \idop+
                    \regop^{-1/2}
                    \sumcov_\iteridx
                    \regop^{-1/2}
                }^{1/2}
            }{\delta}
        }
    }
    \geq
    1-\delta,
    \label{eq:sn-sequence}
\end{equation}
where $\sumcov_\iteridx:=\sum_{i=1}^\iteridx\vncov_i$.
\end{theorem}


The proof extends finite-dimensional self-normalization through finite-rank approximations and continuity of the relevant quadratic forms and Fredholm determinants. The complete argument is given in \autoref{thr:sn-general} and \autoref{thr:sn-general-pd}. The following section applies \autoref{thr:sn-sequence} to regression.


\section{Applications to operator learning}
\label{sec:app}
We now apply \autoref{thr:sn-sequence} to operator regression under sequentially dependent observations. We first consider linear operators, including targets outside the Hilbert estimation space, and then extend the analysis to nonlinear parametric models with strongly convex losses and regularizers.

\subsection{Linear regression in Hilbert spaces}
\label{sec:linear-regression}

Let $\mdlop_\star: \inH\to\outH$ be an unknown bounded linear operator between Hilbert spaces $\inH$ and $\outH$. Suppose we approximate $\mdlop_\star$ by a Hilbert-Schmidt operator $\widehat{\mdlop}_\iteridx \in \Schatten_2(\inH,\outH)$ after $\iteridx\geq 1$ observations. As $\mdlop_\star$ need not be Hilbert-Schmidt, we measure approximation errors through the dual pairing with trace-class operators $\testop\in\Schatten_1(\inH,\outH)$,
\begin{equation}
    \pairing{\mdlop, \testop} := \trace(\mdlop^\adj\testop) = \trace(\testop^\adj \mdlop), \qquad \mdlop\in\lopspace(\inH,\outH).
\end{equation}
For instance, given a rank-one test operator $\testop=\outvec\otimes\invec$, then $\pairing{\targetop-\learntop_\iteridx,\testop} = \inner{\outvec,(\targetop-\learntop_\iteridx)\invec}$, which measures the prediction error at input $\invec$ along the output direction $\outvec$. To avoid clutter, we write $\fspace:=\Schatten_1(\inH,\outH)$ and $\mdlspace:=\Schatten_2(\inH,\outH)$, with $\fspace^\dual$ identified with $\lopspace(\inH,\outH)$ through the trace pairing.

For each $\iteridx\in\N$, let $\obsop_\iteridx:\obsspace_\iteridx\to\fspace$ be a predictable bounded linear map and let $\obsop_\iteridx^\adj:\fspace^\dual\to\obsspace_\iteridx$ denote its Banach adjoint, identifying the dual of the Hilbert space $\obsspace_\iteridx$ with itself through the Riesz representation. We observe $\observation_\iteridx=\obsop_\iteridx^\adj(\targetop)+\vnoise_\iteridx$. For example, $\obsop_\iteridx(\outvec)=(\discretization_\iteridx^\adj\outvec)\otimes\invec_\iteridx$, whose adjoint satisfies $\obsop_\iteridx^\adj(\mdlop)=\discretization_\iteridx\mdlop\invec_\iteridx$. At each round, we estimate $\learntop_\iteridx$ from the accumulated data by minimizing:
\begin{equation}
    \Loss_\iteridx(\mdlop) := \regfactor\norm{\mdlop}_\mdlspace^2 + \sum_{i=1}^\iteridx \norm{\observation_i-\obsop_i^\adj(\mdlop)}_{\obsspace_i}^2, \qquad \mdlop\in\mdlspace,
\end{equation}
where $\regfactor>0$ is a regularization factor. We can now state our result on the linear regression error.

\begin{theorem}[Linear regression error]
\label{thr:linear-regression}
Let $\obsop_\iteridx:\obsspace_\iteridx\to\fspace$ be predictable bounded linear maps and suppose that $\vnoise_\iteridx$ is conditionally $\vncov_\iteridx$-sub-Gaussian, where $\vncov_\iteridx$ is $\filtration_{\iteridx-1}$-measurable and positive semidefinite, for $\iteridx\geq 1$. Assume, for a fixed $\sigma_\vnoise^2>0$, that $\norm{\vncov_\iteridx}_\opnorm\leq\sigma_\vnoise^2$ and that $\obsop_\iteridx\vncov_\iteridx\obsop_\iteridx^\adj$ is almost surely trace class on $\mdlspace$, for all $\iteridx\in\N$. Define $\regsumop_\iteridx:=\regfactor\idop+\sum_{i=1}^\iteridx\obsop_i\obsop_i^\adj$ and $\sumcov_\iteridx:=\sum_{i=1}^\iteridx\obsop_i\vncov_i\obsop_i^\adj$. Then, given $\delta\in(0,1]$, with probability at least $1-\delta$,
\begin{equation}
    \abs{\pairing{\targetop-\learntop_\iteridx,\testop}}
    \leq
    \regfactor\norm{\regsumop_\iteridx^{-1}(\testop)}_1\norm{\targetop}_\opnorm
    +
    \norm{\regsumop_\iteridx^{-\half}(\testop)}_2
    \sqrt{
        2\sigma_\vnoise^2
        \log\paren*{
            \frac{
                \det\paren{
                    \idop+\regfactor^{-1}\sigma_\vnoise^{-2}\sumcov_\iteridx
                }^\half
            }{\delta}
        }
    },
    \label{eq:linear-regression-bound}
\end{equation}
simultaneously for all $\testop\in\Schatten_1(\inH,\outH)$ and $\iteridx\in\N$.
\end{theorem}

To illustrate how \autoref{thr:linear-regression} can yield convergence rates once the data-collection geometry is controlled, consider a simple finite-dimensional persistent-excitation setting.

\begin{corollary}[Illustrative convergence rate]
\label{cor:linear-persistent-excitation}
Under the assumptions of \autoref{thr:linear-regression}, suppose that the observation operators satisfy $\operatorname{Ran}(\obsop_\iteridx)\subseteq\mathcal{S}$ for a fixed $d$-dimensional subspace $\mathcal{S}\subset\fspace$, $\norm{\obsop_\iteridx}_\opnorm\leq L$, and $\sum_{i=1}^\iteridx\obsop_i\obsop_i^\adj\succeq\kappa\iteridx P_{\mathcal{S}}$ for some $\kappa,L>0$ and all sufficiently large $\iteridx$. Then, for every fixed $\testop\in\mathcal{S}$ and $\delta\in(0,1]$, with probability at least $1-\delta$,
\begin{equation*}
    \abs{\pairing{\targetop-\learntop_\iteridx,\testop}}
    =
    O\paren*{
        \sqrt{
            \frac{
                d\log(1+\iteridx)+\log(1/\delta)
            }{
                \iteridx
            }
        }
    },
\end{equation*}
uniformly over sufficiently large $\iteridx$. In particular, for fixed $d$ and $\delta$, the error is $O(\sqrt{\log\iteridx/\iteridx})$.
\end{corollary}

More generally, obtaining rates without the finite-dimensional persistent-excitation assumption requires controlling the variance-like $\norm{\regsumop_\iteridx^{-\half}(\testop)}_2$ and Fredholm log-determinant terms in \eqref{eq:linear-regression-bound}, potentially through operator-valued analogues of the maximum information gain \citep{Vakili2021}.

\subsection{Nonlinear regression with parametric models}
\label{sec:nonlinear-regression}

We now approximate $\mdlop_\star:\infspace\to\outfspace$ by a nonlinear parametric model $\mdlop:\infspace\times\paramspace\to\outfspace$, such as a neural operator. Let $\mdlspace$ be a separable Hilbert space containing the model class $\braces{\mdlop(\pch,\parameter)}_{\parameter\in\paramspace}$. Assume evaluations are given by $\obsop_\iteridx^\adj(\mdlop)$, where $\obsop_\iteridx:\obsspace_\iteridx\to\mdlspace$ is a Hilbert-Schmidt operator, for $\mdlop\in\mdlspace$. Given pointwise losses $\loss_\iteridx: \obsspace_\iteridx \times \obsspace_\iteridx \to \R$ and predictable regularizers $R_\iteridx:\mdlspace\to\R$, we estimate:
\begin{equation}
    \widehat{\parameter}_\iteridx
    \in
    \argmin_{\parameter\in\paramspace}
    \left\{
        \sum_{i=1}^\iteridx
        \loss_i\bigl(\obsop_i^\adj(\mdlop(\cdot,\parameter)),\observation_i\bigr)
        +
        R_\iteridx\bigl(\mdlop(\cdot,\parameter)\bigr)
    \right\}, \qquad \iteridx\in\N.
    \label{eq:nonlinear-parametric-estimator}
\end{equation}

\begin{theorem}[Nonlinear parametric regression]
\label{thr:nonlinear-regression}
For each $\iteridx\in\N$, assume that (i) $\loss_\iteridx(\cdot,\observation_\iteridx)$ is twice Fr\'echet differentiable and $\alpha$-strongly convex, (ii) $R_\iteridx$ is Fr\'echet differentiable and $\regfactor_\iteridx$-strongly convex, with $\regfactor_\iteridx\geq\regfactor_0>0$, and (iii) $\vnoise_\iteridx:=\nabla_1\loss_\iteridx(\obsop_\iteridx^\adj(\mdlop_\star),\observation_\iteridx)$ is conditionally $\vncov_\iteridx$-sub-Gaussian, with $\norm{\vncov_\iteridx}_\opnorm\leq\sigma_\vnoise^2$. Suppose that $\mdlop_\star=\mdlop(\cdot,\parameter_\star)$, for some $\parameter_\star\in\paramspace$, and that $\widehat{\parameter}_\iteridx$ is a global solution of \eqref{eq:nonlinear-parametric-estimator}. Then, for every $\delta\in(0,1]$, with probability at least $1-\delta$, simultaneously for every $\iteridx\in\N$,
\begin{equation}
    \begin{split}
        \norm{\mdlop(\pch,\widehat{\parameter}_\iteridx)-\mdlop_\star}_{\operator{H}_\iteridx}
        \leq
        2\norm{
            \nabla R_\iteridx(\mdlop_\star)
        }_{\operator{H}_\iteridx^{-1}}
        +
        \frac{2\sigma_\vnoise}{\sqrt{\alpha}}
        \sqrt{
            2\log\paren*{
                \frac{
                    \det\paren{
                        \idop+\regfactor_0^{-1}\alpha\obsops\obsops^\adj
                    }^\half
                }{\delta}
            }
        }
    \end{split}
    \label{eq:nonlinear-regression-error}
\end{equation}
where $\operator{H}_\iteridx:=\regfactor_\iteridx\idop+\alpha\sum_{i=1}^\iteridx\obsop_i\obsop_i^\adj$.
\end{theorem}

The result follows from \autoref{lem:loss-difference}, global optimality over the parametric model class, and the self-normalized control of the cumulative gradient noise $\sum_{i=1}^\iteridx\obsop_i\vnoise_i$. The factor of two arises because the parametric estimator need not minimize the lifted loss over the whole space $\mdlspace$.

For a practical regularizer, we extend the model-based RKHS construction of \citet{Oliveira2026kernel} to operator-valued predictions. Let $\paramrkhs$ be an RKHS over $\paramspace$ with kernel $\paramkernel: \paramspace\times\paramspace\to\lopspace(\outfspace)$, and suppose that $\mdlop(\infun, \pch): \paramspace\to\outfspace$ lie in $\paramrkhs$, for each $\infun\in\infspace$. Then, letting $\Psi_\infun: \paramrkhs\to\outfspace$ represent the model evaluation at $\infun\in\infspace$, such that $\Psi_\infun\paramkernel(\pch, \parameter) = \mdlop(\infun,\parameter)$,  it follows that $\opkernel_{\mdlop}(\infun,\infun'):=\Psi_{\infun}\Psi_{\infun'}^\adj$ is an operator-valued kernel whose RKHS $\mdlspace$ contains the model class $\braces{\mdlop(\pch,\parameter)}_{\parameter\in\paramspace}$.

Given a predictable random initialization $\parameter_{\iteridx,0}$, one may take $R_\iteridx(\mdlop):=(\regfactor_\iteridx/2)\norm{\mdlop-\mdlop(\cdot,\parameter_{\iteridx,0})}_{\mdlspace}^2$. This satisfies the preceding assumptions and gives $\norm{\nabla R_\iteridx(\mdlop_\star)}_{\operator{H}_\iteridx^{-1}}\leq\sqrt{\regfactor_\iteridx}\norm{\mdlop_\star-\mdlop(\cdot,\parameter_{\iteridx,0})}_{\mdlspace}$. Moreover, the latter distance is bounded by $\norm{\paramkernel(\cdot,\parameter_\star)-\paramkernel(\cdot,\parameter_{\iteridx,0})}_{\paramrkhs}$, allowing the initialization term in \eqref{eq:nonlinear-regression-error} to be controlled through the parameter kernel and the loss \eqref{eq:nonlinear-parametric-estimator} to be practically implemented. For linear kernels, for instance, $R_t(\mdlop_\parameters) = \norm{\paramkernel(\cdot,\parameter_\star)-\paramkernel(\cdot,\parameter_{\iteridx,0})}_{\paramrkhs}^2 = \norm{\parameter-\parameter_{\iteridx,0}}_2^2$, recovering classic L2 regularization, and a squared-exponential $\opkernel_\parameters$ yields $R_t(\mdlop_\parameters) = 2 - 2e^{-\gamma\norm{\parameter-\parameter_{\iteridx,0}}_2^2}$, $\gamma > 0$.

\section{Discussion}
We presented a general concentration framework for sequential operator learning under dependent, adaptive, and possibly infinite-dimensional observations. A key feature is that Hilbert-valued noise is controlled through operator-valued sub-Gaussian covariance proxies, allowing its spectral structure to enter the confidence bounds. In linear regression, this relaxes restrictions arising from scalar-proxy sub-Gaussian analyses \citep[e.g.,][]{Chowdhury2021, Wang2026bovvrkhs}: rather than requiring Hilbert-Schmidt observation operators, it suffices that the induced proxies $\obsop_\iteridx\vncov_\iteridx\obsop_\iteridx^\adj$ are trace class. This includes conditional mean embedding regression with $\operator{K}(x,x')=k(x,x')\idop$ \citep{Grunewalder2012}, whose Gram operator is not trace class for infinite-dimensional outputs.

Self-normalized bounds of this kind underpin confidence sets, regret bounds, and convergence guarantees throughout sequential learning \citep{Abbasi-Yadkori2011,Chowdhury2017,Chowdhury2021,Wang2026bovvrkhs}. Our illustrative corollary shows how the linear bound yields convergence along persistently excited directions, while more general rates are possible via, e.g., operator-valued extensions of Gaussian-process information-gain bounds \citep{Vakili2021}. The nonlinear result is more preliminary, as pointwise confidence bounds additionally require controlling a variance-like evaluation term. These developments are particularly relevant to adaptive experimental design and Bayesian optimization with neural operators, and may also support learning guarantees from stochastic dynamical data under suitable excitation or dependence assumptions.

\bibliographystyle{plainnat}
\bibliography{mendeley,extra}


\newpage
\appendix

\section{Further background}
\label{app:bg}

\subsection{Definitions}
\label{app:defs}
\begin{definition}[$p$-Schatten norm]
    \label{def:schatten}
    Let $\anyoperator: \inH \to \outH$ be a bounded linear operator between separable Hilbert spaces $\inH$ and $\outH$. Given $1 \leq p < \infty$, the Schatten $p$-norm of $\anyoperator$ is defined as:
    \begin{equation}
        \norm{\anyoperator}_\psnorm := (\trace((\anyoperator^\adj \anyoperator)^{p/2}))^{\frac{1}{p}}.
    \end{equation}
    For $p=\infty$, we set $\norm{\anyoperator}_\infty := \norm{\anyoperator}_\opnorm$. The space of all bounded linear operators between $\inH$ and $\outH$ with finite Schatten $p$-norm is denoted by $\Schatten_p(\inH, \outH)$.
\end{definition}
For any $p \geq 1$, the Schatten $p$-class $\Schatten_p(\inH, \outH)$ is a Banach space under the Schatten $p$-norm. As special cases, $\Schatten_1(\inH, \outH)$ corresponds to the space of all trace-class operators on $\Hspace$ equipped with the trace (or nuclear) norm. For $p=2$, $\Schatten_2(\inH, \outH)$ is the space of Hilbert-Schmidt operators on $\Hspace$, which is actually a Hilbert space equipped with the inner product $\inner{\anyoperator, \anotheroperator} := \trace(\anyoperator^\adj\anotheroperator)$. As $p\to\infty$, $\Schatten_\infty(\inH, \outH)$ is the Banach space of compact operators on $\Hspace$ equipped with $\norm{\cdot}_\infty = \norm{\cdot}_\opnorm$. Therefore, we have the set inclusion:
\begin{equation}
    \Schatten_1(\inH, \outH) \subset \Schatten_2(\inH, \outH) \subset \Schatten_\infty(\inH, \outH) \subset \lopspace(\inH,\outH),
\end{equation}
or more generally $\Schatten_p(\inH, \outH) \subset \Schatten_q(\inH, \outH)$, for any $1\leq p < q \leq \infty$.
Schatten $p$-norms are also unitarily invariant \citep{Kuroda1958} and obey an extension of H\"older's inequality \citep{Simon1977}:
\begin{equation}
    \abs{\trace(\anyoperator\anotheroperator)} \leq \norm{\anyoperator}_p \norm{\anotheroperator}_q,
    \label{eq:schatten-holder}
\end{equation}
for $\anyoperator \in \Schatten_p(\outH, \fspace)$ and $\anotheroperator \in \Schatten_q(\inH, \outH)$ with $p, q \in [1, \infty]$ such that $p^{-1} + q^{-1} = 1$, where $\fspace$ is another separable Hilbert space. In particular, for $\anyoperator \in \lopspace(\outH, \fspace)$ and $\anotheroperator\in\Schatten_1(\inH, \outH)$, it holds that \citep[Thm. 18.11]{Conway2000}:
\begin{equation}
    \abs{\trace(\anyoperator\anotheroperator)} \leq \norm{\anyoperator}_\opnorm \norm{\anotheroperator}_1.
\end{equation}
For more details about Schatten $p$-classes and $p$-norms, the reader is referred to standard operator theory textbooks \citep[e.g.,][]{Weidmann1980, Conway2000, Simon2010}.

\begin{definition}[Fredholm determinant]
    \label{def:fredholm-det}
    Let $\anyoperator\in\Schatten_1(\Hspace)$ be a trace-class operator on a separable Hilbert space $\Hspace$. The Fredholm determinant $\det(\idop + \anyoperator)$ is given by:
    \begin{equation}
        \det(\idop + \anyoperator) = \prod_{i=1}^\infty (1 + \eigval_i(\anyoperator)),
    \end{equation}
    where $\eigval_i(\anyoperator)$ denotes the $i$th eigenvalue of $\anyoperator$, including multiplicities and zeros.
\end{definition}

The Fredholm determinant has the following basic properties \citep{Simon1977}:
\begin{alignat}{2}
    &\llap{\text{(Series)}\quad} &\det(\idop + \anyoperator) &= \exp\paren*{\sum_{\seqidx=1}^\infty \frac{(-1)^\seqidx}{\seqidx}\trace(\anyoperator^\seqidx)} \label{eq:det-exp}\\
    &\llap{\text{(Continuity)}\quad} &\norm{\anyoperator_\seqidx - \anyoperator}_1 \xrightarrow{\seqidx\to\infty} 0 &\implies \det(\idop + \anyoperator_\seqidx) \xrightarrow{\seqidx\to\infty} \det(\idop + \anyoperator), \label{eq:fredholm-continuity}
\end{alignat}
for a trace-class operator $\anyoperator$ and a sequence of trace-class operators $\{\anyoperator_\seqidx\}_{\seqidx=1}^\infty$. A few determinant identities from the finite-dimensional case can be easily extended to the infinite-dimensional setting by taking limits of finite-dimensional matrix sequences (see \autoref{thr:fredholm-finite}) due to the continuity of the Fredholm determinant under the trace norm \eqref{eq:fredholm-continuity}. A useful consequence is an extension of Sylvester's determinant identity to infinite-dimensional operators (see \citealp[Eq. 9.37]{Schmudgen2012}, or \citealp[Ch. IV]{Gohberg1969}):
\begin{equation}
    \det(\idop + \anyoperator\anotheroperator) = \det(\idop + \anotheroperator\anyoperator), \quad \anyoperator\in\Schatten_1(\Hspace), \quad \anotheroperator \in \lopspace(\Hspace).
    \label{eq:fredholm-sylvester}
\end{equation}

\begin{definition}[Conditionally sub-Gaussian vector]
    \label{def:sub-gaussian}
    Let $\outH$ be a Hilbert space, $\{\filtration_\iteridx\}_{\iteridx=0}^\infty$ a filtration, and let $\{\vnoise_\iteridx\}_{\iteridx\in\N}$ denote a sequence of $\outH$-valued random variables adapted to $\{\filtration_\iteridx\}_{\iteridx=1}^\infty$. Then $\vnoise_\iteridx$ is said to be conditionally $\vncov$-sub-Gaussian, given a bounded positive-semidefinite linear operator $\vncov$ on $\outH$, if:
    \begin{equation}
        \forall \iteridx\geq 1, \quad \forall \outfun \in \outH, \quad \expectation[\exp(\inner{\outfun, \vnoise_\iteridx}) \mid \filtration_{\iteridx-1}] \leq \exp\left(\frac{1}{2}\inner{\outfun, \vncov \outfun} \right) \as.
        \label{eq:v-sub-g-op}
    \end{equation}
\end{definition}
Note that this definition encompasses the common characterization of sub-Gaussianity with respect to a scalar parameter $\sigma_\vnoise^2 > 0$ \citep{Abbasi-Yadkori2011, Chowdhury2017}, i.e.:
\begin{equation}
    \forall \iteridx\geq 1, \quad \forall \outfun \in \outH, \quad \expectation[\exp(\inner{\outfun, \vnoise_\iteridx}) \mid \filtration_{\iteridx-1}] \leq \exp\left(\frac{1}{2}\norm{\outfun}^2\sigma_\vnoise^2\right) \as\,,
    \label{eq:v-sub-g}
\end{equation}
by setting $\Sigma := \sigma_\obsnoise^2 \idop$, so that $\sigma_\vnoise^2 := \norm{\vncov}_\opnorm$, as $\inner{\outfun, \vncov \outfun} \leq \norm{\outfun}^2\norm{\vncov}_\opnorm$. This definition is sometimes referred to as ``weakly'' sub-Gaussian, given that there is no guarantee that $\norm{\rv}_\outH < \infty$, though $\inner{\rv, \outvec}_\outH$ is $\sigma_\vnoise^2\norm{\outvec}_\outH^2$-sub-Gaussian for all $\outvec\in\outH$ \citep{Giorgobiani2020}. However, the definition of sub-Gaussianity w.r.t.\ an operator is more general \citep{Antonini1997}. As a basic consequence of the definition, it follows that, if $\rv$ is $\vncov$-sub-Gaussian, then $\anyoperator\rv$ is $\anyoperator\vncov\anyoperator^\adj$-sub-Gaussian, for $\anyoperator\in\lopspace(\outH,\inH)$ \citep{Mollenhauer2023}. Indeed, given any $\invec\in\inH$, by \autoref{def:sub-gaussian}, it holds that:
\begin{equation}
    \Ex{\exp(\inner{\anyoperator\rv, \invec}_\inH)} = \Ex{\exp(\inner{\rv, \anyoperator^\adj\invec}_\outH)} \leq \exp\paren*{\fhalf \inner{\anyoperator^\adj\invec, \vncov\anyoperator^\adj\invec}} = \exp\paren*{\fhalf \inner{\invec, \anyoperator\vncov\anyoperator^\adj\invec}}.
    \label{eq:sub-g-linear}
\end{equation}

\subsection{Existing results}
\begin{lemma}[{\citealp[Cor. 9.14]{Schmudgen2012}}]
    \label{thr:fredholm-finite}
    Let $\{\onb_n\}_{n=1}^\infty$ be an orthonormal basis of a separable Hilbert space $\Hspace$. Then,
    \begin{equation*}
        \det(\idop + \anyoperator) = \lim_{n \to \infty} \det\paren{[\delta_{ij} + \inner{\anyoperator\onb_i, \onb_j}]_{i,j=1}^n}, \quad \anyoperator \in \Schatten_1(\Hspace),
    \end{equation*}
    where $\delta_{ij}$ denotes the Kronecker delta.
\end{lemma}

\begin{lemma}[{Weyl-von Neumann theorem, \citealp{Kuroda1958}}]
    \label{thr:kuroda}
    Let $\anyoperator$ be a bounded, self-adjoint operator on a separable Hilbert space $\Hspace$. Then, for any $1< p < \infty$ and $\epsilon > 0$, there exists a self-adjoint operator $\anotheroperator_\epsilon$ such that $\| \anotheroperator_\epsilon \|_p \leq \epsilon$ and the self-adjoint operator $\anyoperator + \anotheroperator_\epsilon$ has pure point spectrum.
\end{lemma}

\begin{lemma}[{\citealp[Thm. 1]{Pena2009article}}]
    \label{thr:sn-pena}
    Let $\vec{\rv}$ be a random vector in $\R^\anydim$ and $\paramcov$ be a positive-definite random matrix in $\R^{\locdim\times\anydim}$ such that:
    \begin{equation}
        \forall \anyvec \in \R^\locdim, \quad \Ex{\exp\left(\anyvec^\transpose \vec\rv - \fhalf\anyvec^\transpose \paramcov \anyvec\right)} \leq 1.
    \end{equation}
    Then, for any given fixed positive-definite matrix $\regmat$,
    \begin{align}
        &\Ex{
            \sqrt{\frac{\det\regmat}{\det(\paramcov + \regmat)}}
            \exp
            \left(
                \fhalf \vec\rv^\transpose (\paramcov + \regmat)^{-1} \vec\rv
            \right)
        }
        \leq 1,\\
        &\Ex{\exp\left(\frac{1}{4}\vec\rv^\transpose (\paramcov + \regmat)^{-1}\vec\rv\right)} \leq \sqrt{\Ex{\sqrt{\det(\eye + \regmat^{-1}\paramcov}}}.
    \end{align}
\end{lemma}

\paragraph{Linear algebra identities.} We will frequently use linear algebra identities which are well known for finite-dimensional matrices and extensible to the case of bounded linear operators under mild conditions regarding their inverses. Namely, we will use the following form of Woodbury's identity:
\begin{equation}
    (\anyoperator + \anotheroperator\anotheroperator^\adj)^{-1} = \anyoperator^{-1} - \anyoperator^{-1}\anotheroperator(\idop + \anotheroperator^\adj \anyoperator^{-1} \anotheroperator)^{-1}\anotheroperator^\adj\anyoperator^{-1},
    \label{eq:woodbury}
\end{equation}
which holds for any $\anotheroperator \in \lopspace(\inH,\outH)$ and boundedly invertible $\anyoperator \in \lopspace(\outH)$, where $\inH$ and $\outH$ are Hilbert spaces over the same scalar field. See \citet{Deng2011} for a general version. We will also use the following ``push-through'' identity \citep{Henderson1981pushthrough}, which is readily applicable to the linear operator case:
\begin{equation}
    (\anyoperator + \anotheroperator\anotheroperator^\adj)^{-1}\anotheroperator = \anyoperator^{-1}\anotheroperator (\idop + \anotheroperator^\adj\anyoperator^{-1}\anotheroperator)^{-1},
    \label{eq:push-through}
\end{equation}
under the same assumptions on $\anyoperator$ and $\anotheroperator$ as \autoref{eq:woodbury}. This identity \eqref{eq:push-through} can be easily verified by an application of Woodbury's identity, followed by simple linear algebra manipulations.

\section{Auxiliary results}
\label{app:aux}

We here present auxiliary results we derived to construct the critical arguments for the proofs of our main theoretical results. Some of these results, such as \autoref{thr:sn-general}'s extension of \citeauthor{Pena2009article}'s self-normalized concentration to Hilbert-valued noise, might be of independent interest.

\begin{theorem}[Self-normalized concentration in Hilbert spaces]
\label{thr:sn-general}
    Let $\Hspace$ be a separable real Hilbert space, $\rv$ be a $\Hspace$-valued random vector and $\vncov$ be a positive-semidefinite, trace-class random operator on $\Hspace$ such that:
    \begin{equation}
        \forall \anyfunction\in\Hspace, \quad \Ex{\exp\left(\inner{\anyfunction, \rv} - \fhalf \inner{\anyfunction, \vncov \anyfunction}\right)} \leq 1.
        \label{eq:sn-sg-condition}
    \end{equation}
    Consider a fixed, boundedly invertible, positive-definite diagonalizable \footnote{Recall that an operator $\anyoperator\in\lopspace(\Hspace)$ is diagonalizable if there exists an orthonormal basis $\{\eigvec_\seqidx\}_{\seqidx=1}^\infty$ of $\Hspace$ and a scalar sequence $\seq\eigval$ such that $\anyoperator = \sum_{\seqidx=1}^\infty \eigval_\seqidx \eigvec_\seqidx \otimes \eigvec_\seqidx$, where the series converges in the strong operator topology.}
    operator $\regop:\Hspace\to\Hspace$, i.e., there exists an orthonormal basis $\{\onb_i\}_{i=1}^\infty$ of $\Hspace$ such that $\regop \onb_i = \eigval_i \onb_i$, for some $\eigval_i > 0$, for all $i\in\N$, and $\inf_{i\in\N}\eigval_i > 0$. It then holds that:
    \begin{equation}
        \Ex{
            \frac{1}{\sqrt{\det(\idop + \regop^{-1} \vncov)}}
            \exp
            \left(
                \fhalf \norm{\rv}_{(\vncov + \regop)^{-1}}^2
            \right)
        }
        \leq 1.
    \end{equation}
\end{theorem}
\begin{proof}
    We will start the proof with a restriction to the finite-dimensional case via orthogonal projections. This reduction will allow us to apply \autoref{thr:sn-pena} to finite-dimensional subspaces of $\Hspace$. We will then lift it back up to the infinite-dimensional $\Hspace$  by taking appropriate limits of the operators and functions involving them.

    We now construct the finite-dimensional restriction. Given the orthonormal basis $\{\onb_i\}_{i=1}^\infty$ of $\Hspace$ that diagonalizes $\regop$, let $\basisop_\anydim := [\onb_1, \dots, \onb_\anydim]$ represent the operator mapping $\anyvec = [\anyvecelem_i]_{i=1}^\anydim\in\R^\anydim$ to $\basisop_\anydim \anyvec = \sum_{i=1}^\anydim \anyvecelem\onb_i \in \Hspace$, for any given $\anydim\in\N$. Similarly, its adjoint $\basisop_\anydim^\adj$ maps $\anyfunction\in\Hspace$ to the vector $\basisop_\anydim^\adj\anyfunction = \vec\anyfunction_\anydim := [\inner{\anyfunction,\onb_i}]_{i=1}^\anydim\in\R^\anydim$. Combining the two, we have that $\proj_\anydim := \basisop_\anydim \basisop_\anydim^\adj$ forms a finite-rank orthogonal projection onto a $\anydim$-dimensional subspace $\Hspace_\anydim = \proj_\anydim(\Hspace) \subset \Hspace$. Note that, for any $\anyfunction\in\Hspace_\anydim$, we have $\proj_\anydim\anyfunction = \anyfunction$. As \autoref{eq:sn-sg-condition} holds over all $\Hspace$, it also holds over $\Hspace_\anydim$, and we have that:
    \begin{equation}
        \begin{split}
            \forall \anyfunction \in\Hspace_\anydim, \quad 1
            &\geq \Ex{\exp\left(\inner{\anyfunction, \rv} - \fhalf \inner{\anyfunction, \vncov \anyfunction}\right)}\\
            &= \Ex{\exp\left(\inner{\proj_\anydim\anyfunction, \rv} - \fhalf \inner{\proj_\anydim\anyfunction, \vncov \proj_\anydim\anyfunction}\right)}\\
            &= \Ex{\exp\left(\inner{\basisop_\anydim^\adj\anyfunction, \basisop_\anydim^\adj\rv} - \fhalf \inner{\basisop_\anydim^*\anyfunction, \basisop_\anydim^*\vncov_\anydim \basisop_\anydim\basisop_\anydim^*\anyfunction}\right)}\\
            &= \Ex{\exp\left(\vec\anyfunction_\anydim^\transpose \vec\rv_\anydim - \fhalf \vec\anyfunction_\anydim^\transpose\mat\vncov_\anydim \vec\anyfunction_\anydim\right)},
        \end{split}
    \end{equation}
    where $\vec\anyfunction_\anydim := \basisop_\anydim^\adj\anyfunction \in \R^\anydim$, $\vec\rv_\anydim := \basisop_\anydim^\adj\rv \in \R^\anydim$, and $\mat\vncov_\anydim := \basisop_\anydim^\adj \vncov \basisop_\anydim \in \R^{\anydim\times\anydim}$, noting that the inner product in the third line is the Euclidean inner product in $\R^\anydim$. Since the inequality above holds for any $\anyfunction_\anydim \in \Hspace_\anydim$, we have that:
    \begin{equation}
        \forall \anyvec \in \R^\anydim, \quad \Ex{\exp\left(\anyvec^\transpose\vec\rv_\anydim - \fhalf \anyvec^\transpose \mat\vncov_\anydim \anyvec \right)} \leq 1.
    \end{equation}
    Furthermore, $\mat\vncov_{\anydim, \seqidx} := \mat\vncov_\anydim + \frac{1}{\seqidx}\eye \succ \mat\vncov_\anydim$ is almost surely positive definite, for $\seqidx\in\N$, and:
    \begin{equation}
        \forall \anyvec \in \R^\anydim, \quad
        \Ex{\exp\left(\anyvec^\transpose\vec\rv_\anydim - \fhalf \anyvec^\transpose \mat\vncov_{\anydim,\seqidx} \anyvec \right)}
        \leq
        \Ex{\exp\left(\anyvec^\transpose\vec\rv_\anydim - \fhalf \anyvec^\transpose \mat\vncov_\anydim \anyvec \right)}
        \leq 1.
    \end{equation}
    Therefore, by \autoref{thr:sn-pena}, setting $\regmat_\anydim := \basisop_\anydim^\adj \regop \basisop_\anydim = [\inner{\onb_i, \regop\onb_j}]_{i,j=1}^\anydim \in \R^{\anydim\times\anydim}$, it holds that:
    \begin{equation}
        \Ex{
            \sqrt{\frac{\det\regmat_\anydim}{\det(\mat\vncov_{\anydim,\seqidx} + \regmat_\anydim)}}
            \exp
            \left(
                \fhalf \vec\rv_\anydim^\transpose (\mat\vncov_{\anydim,\seqidx} + \regmat_\anydim)^{-1} \vec\rv_\anydim
            \right)
        }
        \leq 1, \qquad \forall\seqidx\in\N.
    \end{equation}
    Now, letting $\seqidx \to \infty$, by Fatou's lemma \citep[Thm. 1.6.5]{Durrett2019}, we get:
    \begin{equation}
        \begin{split}
            1 &\geq
            \liminf_{\seqidx \to \infty} \Ex{
                \sqrt{\frac{\det\regmat_\anydim}{\det(\mat\vncov_{\anydim,\seqidx} + \regmat_\anydim)}}
                \exp
                \left(
                    \fhalf \vec\rv_\anydim^\transpose (\mat\vncov_{\anydim,\seqidx} + \regmat_\anydim)^{-1} \vec\rv_\anydim
                \right)
            }\\
            &\geq
            \Ex{
                \liminf_{\seqidx \to \infty}
                \sqrt{\frac{\det\regmat_\anydim}{\det(\mat\vncov_{\anydim,\seqidx} + \regmat_\anydim)}}
                \exp
                \left(
                    \fhalf \vec\rv_\anydim^\transpose (\mat\vncov_{\anydim,\seqidx} + \regmat_\anydim)^{-1} \vec\rv_\anydim
                \right)
            }\\
            &= \Ex{
                \sqrt{\frac{\det\regmat_\anydim}{\det(\mat\vncov_\anydim + \regmat_\anydim)}}
                \exp
                \left(
                    \fhalf \vec\rv_\anydim^\transpose (\mat\vncov_\anydim + \regmat_\anydim)^{-1} \vec\rv_\anydim
                \right)
            }.
        \end{split}
        \label{eq:sn-finite-dim}
    \end{equation}

    As the result above holds for arbitrary $\anydim\in\N$, to extend it to the whole of $\Hspace$, we can take the limit as $\anydim\to\infty$ and then apply Fatou's lemma. Firstly, we need to derive the limits of the expressions involving operators. As $\anydim\to\infty$, one can easily verify that $\proj_\anydim$ converges to the identity $\idop$ in the strong operator topology (SOT), i.e., $\norm{\anyfunction - \proj_\anydim \anyfunction} \to 0$, for all $\anyfunction\in\Hspace$. Recalling the definition of a diagonal operator, we know that $\regop = \sum_{i=1}^\infty \eigval_i \onb_i \otimes \onb_i$, where the series converges in SOT, and $\{\eigval_i\}_{i=1}^\infty$ is the sequence of eigenvalues of $\regop$, including their multiplicities, which we assume to be sorted in non-increasing order, i.e., $\eigval_1 \geq \eigval_2 \geq \eigval_3 \geq \dots$. It then follows that:
    \begin{align}
        \regmat_\anydim &= \basisop_\anydim^\adj \regop \basisop_\anydim = \basisop_\anydim^\adj \left(\sum_{i=1}^\infty \eigval_i \onb_i \otimes \onb_i \right) \basisop_\anydim = \sum_{i=1}^\anydim \eigval_i \vec\onb_i \otimes \vec\onb_i\\
        \regmat_\anydim^{-1} &= \sum_{i=1}^\anydim \eigval_i^{-1} \vec\onb_i \otimes \vec\onb_i = \basisop_\anydim^\adj \left(\sum_{i=1}^\infty \eigval_i^{-1} \onb_i \otimes \onb_i \right) \basisop_\anydim = \basisop_\anydim^\adj \regop^{-1} \basisop_\anydim,
    \end{align}
    where $\vec\onb_i = \basisop_\anydim^\adj \onb_i \in \R^\anydim$ is the $i$th standard basis vector of $\R^\anydim$. For the determinant, we then have that:
    \begin{equation}
        \begin{split}
            \frac{\det\regmat_\anydim}{\det(\mat\vncov_\anydim + \regmat_\anydim)}
            &= \frac{1}{\det(\eye + \regmat_\anydim^{-1}\mat\vncov_\anydim)}\\
            &= \frac{1}{\det(\eye + \basisop_\anydim^\adj \regop^{-1} \basisop_\anydim\basisop_\anydim^\adj\vncov\basisop_\anydim)}\\
            &= \frac{1}{\det(\eye + \basisop_\anydim^\adj\basisop_\anydim\basisop_\anydim^\adj \regop^{-1} \vncov\basisop_\anydim)}\\
            &= \frac{1}{\det(\eye + \basisop_\anydim^\adj \regop^{-1} \vncov\basisop_\anydim)}\\
        \end{split}
    \end{equation}
    using the identities $\det (\mat{A} \mat{B}) = \det (\mat{A}) \det (\mat{B})$, for square matrices $\mat{A}$ and $\mat{B}$,
    $\basisop_\anydim\basisop_\anydim^\adj = \proj_\anydim$ and $\basisop_\anydim^\adj \basisop_\anydim = \eye$, and lastly
    $\regop^{-1}\proj_\anydim = \proj_\anydim \regop^{-1}$, since $\proj_\anydim$ and $\regop$ commute, given that $\proj_\anydim$ is formed by eigenvectors of $\regop$.  As $\vncov$ is trace-class and $\regop^{-1}$ is bounded, since $\inf_{i\in\N} \eigval_i > 0$, we also have that $\regop^{-1}\vncov$ is trace-class, so that the Fredholm determinant $\det(\idop + \regop^{-1}\vncov)$ and its approximations with finite-rank projections $\det(\eye + \basisop_\anydim^\adj \regop^{-1} \vncov \basisop_\anydim) = \det([\delta_{ij} + \inner{\regop^{-1} \vncov \onb_i, \onb_j}]_{i,j=1}^\anydim)$ are well defined. Therefore, by \autoref{thr:fredholm-finite}, it follows that:
    \begin{equation}
        \lim_{\anydim\to\infty} \det(\eye + \basisop_\anydim^\adj \regop^{-1} \vncov \basisop_\anydim) = \det(\idop + \regop^{-1} \vncov),
        \label{eq:sn-det-lim}
    \end{equation}
    which holds almost surely with respect to the randomness of $\vncov$.

    For the quadratic term,
    \begin{equation}
        \begin{split}
            \vec\rv_\anydim^\transpose (\mat\vncov_\anydim + \regmat_\anydim)^{-1} \vec\rv_\anydim
            &= \inner{\vec\rv_\anydim, (\mat\vncov_\anydim + \regmat_\anydim)^{-1} \vec\rv_\anydim}\\
            &= \inner{\basisop_\anydim^\adj\rv, (\mat\vncov_\anydim + \regmat_\anydim)^{-1} \basisop_\anydim^\adj\rv}\\
            &= \inner{\rv, \basisop_\anydim(\mat\vncov_\anydim + \regmat_\anydim)^{-1} \basisop_\anydim^\adj\rv}.
        \end{split}
    \end{equation}
    Letting $\regext := \vncov + \regop$ and $\mat{\regext}_\anydim := \mat\vncov_\anydim + \regmat_\anydim = \basisop_\anydim^\adj \regext\basisop_\anydim$, we have that:
    \begin{equation}
        \begin{split}
            \forall \anyfunction \in\Hspace, \quad \basisop_\anydim(\mat\vncov_\anydim + \regmat_\anydim)^{-1} \basisop_\anydim^\adj\regext\anyfunction
            &= \basisop_\anydim\mat{\regext}_\anydim^{-1} \basisop_\anydim^\adj \regext\anyfunction\\
            &= \basisop_\anydim\mat{\regext}_\anydim^{-1} \basisop_\anydim^\adj \regext\proj_\anydim\anyfunction + \basisop_\anydim\mat{\regext}_\anydim^{-1} \basisop_\anydim^\adj \regext(\idop - \proj_\anydim)\anyfunction\\
            &= \basisop_\anydim(\basisop_\anydim^\adj \regext\basisop_\anydim)^{-1} \basisop_\anydim^\adj \regext\basisop_\anydim\basisop_\anydim^\adj\anyfunction + \basisop_\anydim\mat{\regext}_\anydim^{-1} \basisop_\anydim^\adj \regext(\idop - \proj_\anydim)\anyfunction\\
            &= \proj_\anydim \anyfunction + \basisop_\anydim\mat{\regext}_\anydim^{-1} \basisop_\anydim^\adj \regext(\idop - \proj_\anydim)\anyfunction.
        \end{split}
        \label{eq:sn-quadratic-decomp}
    \end{equation}
    As $\anydim\to\infty$, $\proj_\anydim \to \idop$ in SOT, so that $\proj_\anydim\anyfunction \to \anyfunction$ and $(\idop - \proj_\anydim)\anyfunction \to 0$, for all $\anyfunction\in\Hspace$. Additionally, as $\mat\regext_\anydim = \mat\vncov_\anydim + \regmat_\anydim \succeq  \regmat_\anydim$, we have that $\norm{\mat{\regext}_\anydim^{-1}} \leq \norm{\regmat_\anydim^{-1}} = \eigval_\anydim^{-1} \leq \frac{1}{\inf_{i\in\N}\eigval_i} < \infty$, so that $\basisop_\anydim\mat{\regext}_\anydim^{-1} \basisop_\anydim^\adj \regext$ is uniformly bounded over all $\anydim\in\N$, as $\inf_{i\in\N}\eigval_i > 0$. Therefore, the identity in \autoref{eq:sn-quadratic-decomp} leads us to:
    \begin{equation}
        \forall\anyfunction\in\Hspace, \quad \lim_{\anydim\to\infty} \basisop_\anydim\mat{\regext}_\anydim^{-1} \basisop_\anydim^\adj\regext\anyfunction = \anyfunction.
    \end{equation}
    We can then conclude that $\basisop_\anydim\mat{\regext}_\anydim^{-1} \basisop_\anydim^\adj$ converges to $\regext^{-1}$ in SOT, which holds over all of $\Hspace$ for the inverse of $\regext$ is bounded. Hence, it almost surely holds that:
    \begin{equation}
        \lim_{\anydim\to\infty} \vec\rv_\anydim^\transpose (\mat\vncov_\anydim + \regmat_\anydim)^{-1} \vec\rv_\anydim = \inner{\rv, \regext^{-1}\rv} = \inner{\rv, (\vncov + \regop)^{-1}\rv}.
        \label{eq:sn-quadratic-lim}
    \end{equation}
    Applying the limits in \autoref{eq:sn-det-lim} and \ref{eq:sn-quadratic-lim} to \autoref{eq:sn-finite-dim}, we obtain:
    \begin{equation}
        \begin{split}
            1
            &\geq
            \Ex{
                \lim_{\anydim\to\infty}
                \sqrt{\frac{\det\regmat_\anydim}{\det(\mat\vncov_\anydim + \regmat_\anydim)}}
                \exp
                \left(
                    \fhalf \vec\rv_\anydim^\transpose (\mat\vncov_\anydim + \regmat_\anydim)^{-1} \vec\rv_\anydim
                \right)
            }\\
            &=
            \Ex{
                \sqrt{\frac{1}{\det(\idop + \regop^{-1} \vncov)}}
                \exp
                \left(
                    \fhalf \inner{\rv, (\vncov + \regop)^{-1}\rv}
                \right)
            },
        \end{split}
    \end{equation}
    where the inequality holds by Fatou's lemma \citep{Durrett2019}, given the non-negativity of the integrand inside the expectation, which concludes the proof.
\end{proof}

\begin{corollary}[Extension to general regularizers]
    \label{thr:sn-general-pd}
    The result in \autoref{thr:sn-general} continues to hold if we let $\regop$ be any boundedly invertible positive-definite operator on $\Hspace$.
\end{corollary}
\begin{proof}
    For this proof, we apply \citeauthor{Kuroda1958}'s \citeyearpar{Kuroda1958} result on the Weyl-von Neumann theorem (\autoref{thr:kuroda}) to construct a sequence of bounded self-adjoint operators $\regop_\seqidx$ with pure point spectrum, such that $\regop_\seqidx$ converges to $\regop$ as $\seqidx\to \infty$ and show that \autoref{thr:sn-general} holds for the limit. Indeed, by \autoref{thr:kuroda}, for any given $\epsilon > 0$ and $1 < p < \infty$, we can construct a diagonalizable operator $\regop + \anyoperator_\epsilon$, such that $\norm{\anyoperator_\epsilon}_\psnorm < \epsilon$. Consequently, we can build a sequence $\{\regop_\seqidx\}_{\seqidx=1}^\infty$ such that $\regop_\seqidx := \regop + \anyoperator_{\epsilon_\seqidx}$, with $\epsilon_\seqidx \to 0$  as $\seqidx\to\infty$, which converges to $\regop$ in operator norm as:
    \begin{equation}
        \norm{\regop_\seqidx - \regop}_\opnorm \leq \norm{\regop_\seqidx - \regop}_\psnorm \leq \norm{\anyoperator_{\epsilon_\seqidx}}_\psnorm \leq \epsilon_\seqidx \xrightarrow{\seqidx\to\infty} 0.
    \end{equation}
    Moreover, \autoref{thr:kuroda} ensures that $\regop_\seqidx$ is bounded, self-adjoint and has a pure point spectrum, so that it admits an orthonormal eigenbasis, i.e., $\regop_\seqidx$ is diagonalizable. To ensure that $\regop_\seqidx$ remains positive definite, knowing that bounded invertibility of a positive-definite $\regop$ guarantees that $\regop \succeq \anyscalar_0 \idop$, for some $\anyscalar_0 > 0$, we can choose $\epsilon_\seqidx$ such that $0 < \epsilon_\seqidx < \anyscalar_0$ (e.g., $\epsilon_\seqidx := \frac{\anyscalar_0}{2\seqidx}$), for all $\seqidx\in\N$, so that:
    \begin{equation}
        \regop_\seqidx = \regop + \anyoperator_{\epsilon_\seqidx} \succeq (\anyscalar_0 - \epsilon_\seqidx)\idop \succ 0 \implies \norm{\regop_\seqidx^{-1}}_\opnorm \leq \frac{1}{\anyscalar_0 - \epsilon_\seqidx} < \infty, \quad \forall\seqidx\in\N.
    \end{equation}
    Applying \autoref{thr:sn-general}, we then have that:
    \begin{equation}
        \forall \seqidx\in\N, \quad \Ex{
            \sqrt{\frac{1}{\det(\idop + \regop_\seqidx^{-1} \vncov)}}
            \exp
            \left(
                \fhalf \inner{\rv, (\vncov + \regop_\seqidx)^{-1}\rv}
            \right)
        } \leq 1.
        \label{eq:sn-vn}
    \end{equation}
    We now need to verify the limits of the determinant and the quadratic form. Firstly, the determinant $\det(\idop + \regop_\seqidx^{-1} \vncov )$ converges to $\det(\idop + \regop^{-1}\vncov)$ as $\regop_\seqidx^{-1}\vncov$ converges to $\regop^{-1}\vncov$ in trace norm \citep[Thm. 3.5]{Simon1977}. Indeed,
    \begin{equation}
        \begin{split}
            \norm{\regop^{-1}\vncov - \regop_\seqidx^{-1}\vncov}_1
            &= \norm{(\regop^{-1} - \regop_\seqidx^{-1})\vncov}_1\\
            &= \norm{\regop^{-1}(\regop_\seqidx - \regop)\regop_\seqidx^{-1}\vncov}_1\\
            &\leq \norm{\regop^{-1}(\regop_\seqidx - \regop)\regop_\seqidx^{-1}}_\opnorm\norm{\vncov}_1\\
            &\leq \norm{\regop^{-1}}_\opnorm\norm{\regop_\seqidx - \regop}_\opnorm\norm{\regop_\seqidx^{-1}}_\opnorm\norm{\vncov}_1\\
            &\leq \epsilon_\seqidx \norm{\regop^{-1}}_\opnorm\norm{\regop_\seqidx^{-1}}_\opnorm\norm{\vncov}_1 \xrightarrow{\seqidx\to \infty} 0,
        \end{split}
    \end{equation}
    using the fact that $\norm{\anyoperator\anotheroperator}_1 \leq \norm{\anyoperator}_\opnorm\norm{\anotheroperator}_1$ \citep[see, e.g.,][Eq. 1.9]{Simon1977}. Thus,
    \begin{equation}
        \lim_{\seqidx \to \infty} \det(\idop + \regop_\seqidx^{-1} \vncov ) = \det(\idop + \regop^{-1} \vncov),
        \label{eq:lim-vn-det}
    \end{equation}
    which holds almost surely. Secondly, the quadratic term is such that:
    \begin{equation}
        \begin{split}
            \abs{\inner{\rv, (\vncov + \regop)^{-1}\rv} - \inner{\rv, (\vncov + \regop_\seqidx)^{-1}\rv}}
            &= \abs{\inner{\rv, ((\vncov + \regop)^{-1} - (\vncov + \regop_\seqidx)^{-1})\rv}}\\
            &= \abs{\inner{\rv, (\vncov + \regop)^{-1}(\vncov + \regop_\seqidx -(\vncov + \regop))(\vncov + \regop_\seqidx)^{-1}\rv}}\\
            &\leq \norm{\rv} \norm{(\vncov + \regop)^{-1}(\regop_\seqidx - \regop)(\vncov + \regop_\seqidx)^{-1}\rv}\\
            &\leq \norm{\rv}^2 \norm{(\vncov + \regop)^{-1}}_\opnorm \norm{\regop_\seqidx - \regop}_\opnorm \norm{(\vncov + \regop_\seqidx)^{-1}}_\opnorm\\
            &\leq \norm{\rv}^2 \norm{\regop^{-1}}_\opnorm \norm{\regop_\seqidx - \regop}_\opnorm \norm{\regop_\seqidx^{-1}}_\opnorm\\
            &\leq \frac{\epsilon_\seqidx}{(\anyscalar_0 - \epsilon_\seqidx)^2} \norm{\rv}^2 \xrightarrow{\seqidx\to \infty} 0,
        \end{split}
    \end{equation}
    where we applied Cauchy-Schwarz to derive the first inequality and then the observation that $\norm{(\vncov + \regop)^{-1}}_\opnorm \leq \norm{\regop^{-1}}_\opnorm$, as $\vncov \succeq 0$, and likewise for $\norm{(\vncov + \regop_\seqidx)^{-1}}_\opnorm$. Therefore, the following almost surely holds:
    \begin{equation}
        \inner{\rv, (\vncov + \regop_\seqidx)^{-1}\rv} \xrightarrow{\seqidx\to\infty} \inner{\rv, (\vncov + \regop)^{-1}\rv}.
        \label{eq:lim-vn-quadratic}
    \end{equation}
    Combining \autoref{eq:lim-vn-det} and \ref{eq:lim-vn-quadratic} with \ref{eq:sn-vn}, the main result then arises by an application of Fatou's lemma \citep{Durrett2019}.
\end{proof}

\begin{lemma}[Loss difference under operator-valued observations]
\label{lem:loss-difference}
Let $\mdlspace$ be a real Hilbert space and, for each $i\in\range{\iteridx}$, let $\obsop_i:\obsspace_i\to\mdlspace$ be bounded and linear. Suppose that $\loss_i(\cdot,\observation_i)$ is twice Fr\'echet differentiable and $\alpha$-strongly convex for some $\alpha>0$, and that $R_\iteridx:\mdlspace\to\R$ is Fr\'echet differentiable and $\regfactor_\iteridx$-strongly convex for some $\regfactor_\iteridx>0$. Define $\Loss_\iteridx(\mdlop):=\sum_{i=1}^\iteridx\loss_i(\obsop_i^\adj(\mdlop),\observation_i)+R_\iteridx(\mdlop)$, and let $\widehat{\mdlop}_\iteridx$ minimize $\Loss_\iteridx$ over $\mdlspace$. Then, for every $\mdlop\in\mdlspace$,
\begin{equation}
    \frac{1}{2}\norm{\mdlop-\widehat{\mdlop}_\iteridx}_{\operator{H}_\iteridx}^2
    \leq
    \Loss_\iteridx(\mdlop)-\Loss_\iteridx(\widehat{\mdlop}_\iteridx)
    \leq
    \frac{1}{2}\norm{\nabla\Loss_\iteridx(\mdlop)}_{\operator{H}_\iteridx^{-1}}^2,
    \label{eq:loss-difference}
\end{equation}
where $\operator{H}_\iteridx:=\regfactor_\iteridx\idop+\alpha\sum_{i=1}^\iteridx\obsop_i\obsop_i^\adj$.
\end{lemma}

\begin{proof}
For any $\mdlop,\mdlop'\in\mdlspace$, the $\alpha$-strong convexity of $\loss_i(\cdot,\observation_i)$ gives
\begin{align}
    \loss_i(\obsop_i^\adj(\mdlop'),\observation_i)
    \geq{}&
    \loss_i(\obsop_i^\adj(\mdlop),\observation_i)
    +
    \inner{
        \nabla_1\loss_i(\obsop_i^\adj(\mdlop),\observation_i),
        \obsop_i^\adj(\mdlop'-\mdlop)
    }_{\obsspace_i}
    \nonumber\\
    &+
    \frac{\alpha}{2}
    \norm{\obsop_i^\adj(\mdlop'-\mdlop)}_{\obsspace_i}^2.
    \label{eq:loss-strong-convexity}
\end{align}
By the definition of the adjoint, the inner-product term in \eqref{eq:loss-strong-convexity} equals $\inner{\obsop_i\nabla_1\loss_i(\obsop_i^\adj(\mdlop),\observation_i),\mdlop'-\mdlop}_{\mdlspace}$, while its final term equals $\frac{\alpha}{2}\norm{\mdlop'-\mdlop}_{\obsop_i\obsop_i^\adj}^2$. Combining these inequalities with the $\regfactor_\iteridx$-strong convexity of $R_\iteridx$ yields
\begin{equation}
    \Loss_\iteridx(\mdlop')
    \geq
    \Loss_\iteridx(\mdlop)
    +
    \inner{\nabla\Loss_\iteridx(\mdlop),\mdlop'-\mdlop}_{\mdlspace}
    +
    \frac{1}{2}\norm{\mdlop'-\mdlop}_{\operator{H}_\iteridx}^2.
    \label{eq:lifted-loss-strong-convexity}
\end{equation}

Taking $\mdlop=\widehat{\mdlop}_\iteridx$ in \eqref{eq:lifted-loss-strong-convexity} and using the first-order optimality condition $\nabla\Loss_\iteridx(\widehat{\mdlop}_\iteridx)=0$ gives
\begin{equation}
    \Loss_\iteridx(\mdlop')-\Loss_\iteridx(\widehat{\mdlop}_\iteridx)
    \geq
    \frac{1}{2}\norm{\mdlop'-\widehat{\mdlop}_\iteridx}_{\operator{H}_\iteridx}^2,
\end{equation}
which proves the first inequality in \eqref{eq:loss-difference}.

For the second inequality, fix $\mdlop\in\mdlspace$ and take the infimum over $\mdlop'\in\mdlspace$ in \eqref{eq:lifted-loss-strong-convexity}. Writing $\Delta:=\mdlop'-\mdlop$ and using the optimality of $\widehat{\mdlop}_\iteridx$ gives
\begin{align}
    \Loss_\iteridx(\widehat{\mdlop}_\iteridx)
    &\geq
    \Loss_\iteridx(\mdlop)
    +
    \inf_{\Delta\in\mdlspace}
    \left\{
        \inner{\nabla\Loss_\iteridx(\mdlop),\Delta}_{\mdlspace}
        +
        \frac{1}{2}\norm{\Delta}_{\operator{H}_\iteridx}^2
    \right\}
    \nonumber\\
    &=
    \Loss_\iteridx(\mdlop)
    -
    \frac{1}{2}
    \norm{\nabla\Loss_\iteridx(\mdlop)}_{\operator{H}_\iteridx^{-1}}^2.
\end{align}
Indeed, $\operator{H}_\iteridx\succeq\regfactor_\iteridx\idop$ is boundedly invertible, and the infimum is attained at $\Delta=-\operator{H}_\iteridx^{-1}\nabla\Loss_\iteridx(\mdlop)$. Rearranging completes the proof.
\end{proof}

\section{Proofs of main theorems}
\label{app:main-proofs}
This section presents our proofs for the theoretical results presented in the main paper.

\subsection{Proof of \autoref{thr:sn-sequence}}
\begin{proof}
    The proof follows the standard stopping-time argument for self-normalized concentration \citep[e.g.,][]{Abbasi-Yadkori2012,Chugg2025variational}. For each $\anyfunction\in\Hspace$, define $\martingale_0^\anyfunction:=1$ and
    \begin{equation*}
        \martingale_\iteridx^\anyfunction
        :=
        \exp\paren*{
            \inner*{\anyfunction,\sum_{i=1}^\iteridx\rv_i}_\Hspace
            -
            \frac{1}{2}
            \inner{\anyfunction,\sumcov_\iteridx\anyfunction}_\Hspace
        },
        \qquad \iteridx\in\N.
    \end{equation*}
    Since $\sumcov_\iteridx=\sumcov_{\iteridx-1}+\vncov_\iteridx$, conditional sub-Gaussianity gives
    \begin{equation*}
        \begin{split}
            \Ex{\martingale_\iteridx^\anyfunction\given\filtration_{\iteridx-1}}
            &=
            \martingale_{\iteridx-1}^\anyfunction
            \Ex{
                \exp\paren*{
                    \inner{\anyfunction,\rv_\iteridx}_\Hspace
                    -
                    \frac{1}{2}
                    \inner{\anyfunction,\vncov_\iteridx\anyfunction}_\Hspace
                }
                \given\filtration_{\iteridx-1}
            }\\
            &\leq
            \martingale_{\iteridx-1}^\anyfunction.
        \end{split}
    \end{equation*}
    Thus, $\braces{\martingale_\iteridx^\anyfunction}_{\iteridx\geq0}$ is a nonnegative supermartingale for every $\anyfunction\in\Hspace$.

    For $\iteridx\in\N$, let $\anyevent_\iteridx(\delta)$ be the event
    \begin{equation*}
        \norm*{\sum_{i=1}^\iteridx\rv_i}_{(\regop+\sumcov_\iteridx)^{-1}}^2
        >
        2\log\paren*{
            \frac{
                \det\paren*{
                    \idop+\regop^{-\half}\sumcov_\iteridx\regop^{-\half}
                }^\half
            }{\delta}
        },
    \end{equation*}
    and define the first crossing time $\tstop:=\inf\braces{\iteridx\geq1:\anyevent_\iteridx(\delta)\text{ occurs}}$, with $\inf\emptyset:=\infty$. This is a stopping time because $\anyevent_\iteridx(\delta)\in\filtration_\iteridx$. Set $\tstop_\iteridx:=\min\braces{\tstop,\iteridx}$.

    The stopped process $\braces{\martingale_{\tstop_\iteridx}^\anyfunction}_{\iteridx\geq0}$ is also a nonnegative supermartingale. Therefore, for every $\anyfunction\in\Hspace$,
    \begin{equation*}
        \Ex{\martingale_{\tstop_\iteridx}^\anyfunction}
        \leq
        \Ex{\martingale_0^\anyfunction}
        =
        1, \qquad \iteridx\in\N.
    \end{equation*}
    Equivalently, the stopped random vector $\sum_{i=1}^{\tstop_\iteridx}\rv_i$ and covariance proxy $\sumcov_{\tstop_\iteridx}$ satisfy the hypothesis of \autoref{thr:sn-general-pd}. Applying that result yields
    \begin{equation*}
        \Ex{
            \frac{
                \exp\paren*{
                    \frac{1}{2}
                    \norm*{
                        \sum_{i=1}^{\tstop_\iteridx}\rv_i
                    }_{(\regop+\sumcov_{\tstop_\iteridx})^{-1}}^2
                }
            }{
                \det\paren*{
                    \idop+
                    \regop^{-\half}
                    \sumcov_{\tstop_\iteridx}
                    \regop^{-\half}
                }^\half
            }
        }
        \leq 1.
    \end{equation*}
    Hence, Markov's inequality implies
    \begin{equation*}
        \prob{
            \norm*{
                \sum_{i=1}^{\tstop_\iteridx}\rv_i
            }_{(\regop+\sumcov_{\tstop_\iteridx})^{-1}}^2
            >
            2\log\paren*{
                \frac{
                    \det\paren*{
                        \idop+
                        \regop^{-\half}
                        \sumcov_{\tstop_\iteridx}
                        \regop^{-\half}
                    }^\half
                }{\delta}
            }
        }
        \leq\delta.
    \end{equation*}

    By the definition of the first crossing time, $\tstop\leq\iteridx$ if and only if the claimed inequality is violated at time $\tstop_\iteridx$. Consequently, the preceding bound gives $\prob{\tstop\leq\iteridx}\leq\delta$ for every $\iteridx\in\N$. Since the events $\braces{\tstop\leq\iteridx}$ increase to $\braces{\tstop<\infty}$,
    \begin{equation*}
        \prob{\exists\iteridx\in\N:\anyevent_\iteridx(\delta)}
        =
        \prob{\tstop<\infty}
        =
        \lim_{\iteridx\to\infty}\prob{\tstop\leq\iteridx}
        \leq\delta.
    \end{equation*}
    Taking complements proves that the claimed inequality holds simultaneously for every $\iteridx\in\N$ with probability at least $1-\delta$.
\end{proof}

\subsection{Proof of {\autoref{thr:linear-regression}}}

\begin{proof}
We start with the derivation of a closed-form expression for the least-squares estimator and then proceed to analyse its approximation error. Since $\obsop_\iteridx:\obsspace_\iteridx\to\fspace\subset\mdlspace$, its Banach adjoint restricted to $\mdlspace$ coincides with its Hilbert adjoint. Indeed, for any $\mdlop\in\mdlspace$ and $\outvec\in\obsspace_\iteridx$,
\[
    \inner{\obsop_\iteridx^\adj(\mdlop),\outvec} = \pairing{\mdlop,\obsop_\iteridx(\outvec)} = \inner{\mdlop,\obsop_\iteridx(\outvec)}.
\]
Differentiating the least-squares objective over $\mdlspace$ therefore gives $\nabla\Loss_\iteridx(\mdlop)=2\regsumop_\iteridx\mdlop-2\sum_{i=1}^\iteridx\obsop_i\observation_i$. Since $\regsumop_\iteridx\succeq\regfactor\idop$, the minimizer is unique and satisfies:
\begin{equation}
    \learntop_\iteridx = \regsumop_\iteridx^{-1}\sum_{i=1}^\iteridx\obsop_i\observation_i, \qquad \iteridx\geq 1.
    \label{eq:linear-regression-estimator}
\end{equation}

We first note that $\regsumop_\iteridx^{-1}(\testop)\in\fspace$ for every $\testop\in\fspace$. Let $\obsop_{1:\iteridx}:\obsspace_{1:\iteridx}\to\mdlspace$ be defined by $\obsop_{1:\iteridx}(\outvec_1,\ldots,\outvec_\iteridx):=\sum_{i=1}^\iteridx\obsop_i\outvec_i$, so that $\regsumop_\iteridx=\regfactor\idop+\obsop_{1:\iteridx}\obsop_{1:\iteridx}^\adj$. By Woodbury's identity,
\[
    \regsumop_\iteridx^{-1}
    =
    \regfactor^{-1}\idop
    -
    \regfactor^{-1}\obsop_{1:\iteridx}
    \paren{
        \regfactor\idop+\obsop_{1:\iteridx}^\adj\obsop_{1:\iteridx}
    }^{-1}
    \obsop_{1:\iteridx}^\adj.
\]
Since $\obsop_{1:\iteridx}$ maps $\obsspace_{1:\iteridx}$ into $\fspace$, both terms on the right map $\fspace$ into $\fspace$. Hence $\regsumop_\iteridx^{-1}(\testop)\in\fspace$.

Substituting $\observation_i=\obsop_i^\adj(\targetop)+\vnoise_i$ into \eqref{eq:linear-regression-estimator}, define $\sumrv_\iteridx:=\sum_{i=1}^\iteridx\obsop_i\vnoise_i$. For any $\testop\in\fspace$, since $\regsumop_\iteridx^{-1}(\testop)\in\fspace$, self-adjointness of $\regsumop_\iteridx^{-1}$ on $\mdlspace$ yields:
\begin{align}
\pairing{\targetop-\learntop_\iteridx,\testop}
&=
\pairing{\targetop,\testop}
-\sum_{i=1}^\iteridx
\inner{
    \obsop_i\obsop_i^\adj(\targetop),
    \regsumop_\iteridx^{-1}(\testop)
}
-\inner{
    \sumrv_\iteridx,
    \regsumop_\iteridx^{-1}(\testop)
}\\
&=
\pairing{\targetop,\testop}
-\sum_{i=1}^\iteridx
\pairing{
    \targetop,
    \obsop_i\obsop_i^\adj
    \regsumop_\iteridx^{-1}(\testop)
}
-\inner{
    \sumrv_\iteridx,
    \regsumop_\iteridx^{-1}(\testop)
}\\
&=
\pairing{
    \targetop,
    \paren*{
        \idop-\sum_{i=1}^\iteridx\obsop_i\obsop_i^\adj
        \regsumop_\iteridx^{-1}
    }(\testop)
}
-\inner{
    \sumrv_\iteridx,
    \regsumop_\iteridx^{-1}(\testop)
}\\
&=
\regfactor
\pairing{
    \targetop,
    \regsumop_\iteridx^{-1}(\testop)
}
-\inner{
    \sumrv_\iteridx,
    \regsumop_\iteridx^{-1}(\testop)
},
\label{eq:linear-regression-error-decomp}
\end{align}
where the second equality follows from the defining relation of $\obsop_i^\adj$, and the last follows from the definition $\regsumop_\iteridx=\regfactor\idop+\sum_{i=1}^\iteridx\obsop_i\obsop_i^\adj$. Trace duality and Cauchy-Schwarz then lead us to:
\begin{equation}
    \forall\iteridx\in\N, \quad \abs{\pairing{\targetop-\learntop_\iteridx,\testop}}
    \leq
    \regfactor\norm{\targetop}_\opnorm\norm{\regsumop_\iteridx^{-1}(\testop)}_1
    +
    \norm{\regsumop_\iteridx^{-\half}(\testop)}_2
    \norm*{
        \sum_{i=1}^\iteridx\obsop_i\vnoise_i
    }_{\regsumop_\iteridx^{-1}}, \quad \forall \testop\in\fspace.
    \label{eq:linear-regression-preconcentration}
\end{equation}

Let $\rv_\iteridx:=\obsop_\iteridx\vnoise_\iteridx\in\mdlspace$. By conditional $\vncov_\iteridx$-sub-Gaussianity, $\rv_\iteridx$ is conditionally $\obsop_\iteridx\vncov_\iteridx\obsop_\iteridx^\adj$-sub-Gaussian, which is trace class by assumption. Moreover, $\sumcov_\iteridx\preceq\sigma_\vnoise^2\sum_{i=1}^\iteridx\obsop_i\obsop_i^\adj$, and hence $\sigma_\vnoise^2\regfactor\idop+\sumcov_\iteridx\preceq\sigma_\vnoise^2\regsumop_\iteridx$. Therefore,
\[
    \norm*{
        \sum_{i=1}^\iteridx\obsop_i\vnoise_i
    }_{\regsumop_\iteridx^{-1}}^2
    \leq
    \sigma_\vnoise^2
    \norm*{
        \sum_{i=1}^\iteridx\obsop_i\vnoise_i
    }_{\paren{\sigma_\vnoise^2\regfactor\idop+\sumcov_\iteridx}^{-1}}^2.
\]
Finally, applying \autoref{thr:sn-sequence} with $\regop=\sigma_\vnoise^2\regfactor\idop$, we obtain the following:
\begin{equation}
    \forall\iteridx\in\N, \quad \norm*{
        \sum_{i=1}^\iteridx\obsop_i\vnoise_i
    }_{\regsumop_\iteridx^{-1}}
    \leq
    \sqrt{
        2\sigma_\vnoise^2
        \log\paren*{
            \frac{
                \det\paren{
                    \idop+\regfactor^{-1}\sigma_\vnoise^{-2}\sumcov_\iteridx
                }^\half
            }{\delta}
        }
    },
    \label{eq:linear-regression-noise-bound}
\end{equation}
which holds with probability at least $1-\delta$. Combining \eqref{eq:linear-regression-preconcentration} and \eqref{eq:linear-regression-noise-bound} proves \eqref{eq:linear-regression-bound}. Since the high-probability event in \eqref{eq:linear-regression-noise-bound} does not depend on $\testop$, the result holds simultaneously for every $\testop\in\fspace$.
\end{proof}

\subsection{Proof of \autoref{cor:linear-persistent-excitation}}

\begin{corollary}[Extended restatement of \autoref{cor:linear-persistent-excitation}]
Let the assumptions of \autoref{thr:linear-regression} hold and define $\operator{G}_\iteridx:=\sum_{i=1}^\iteridx\obsop_i\obsop_i^\adj$. Suppose there exists a fixed $d$-dimensional subspace $\mathcal{S}\subset\fspace$, constants $L,\kappa>0$, and $\iteridx_0\in\N$ such that, almost surely, $\operatorname{Ran}(\obsop_\iteridx)\subseteq\mathcal{S}$ and $\norm{\obsop_\iteridx}_\opnorm\leq L$ for every $\iteridx\in\N$, while:
\begin{equation}
    \operator{G}_\iteridx
    \succeq
    \kappa\iteridx P_{\mathcal{S}},
    \qquad
    \iteridx\geq\iteridx_0,
    \label{eq:persistent-excitation}
\end{equation}
where $P_{\mathcal{S}}$ denotes the orthogonal projection onto $\mathcal{S}$ in $\mdlspace$. Let $C_{\mathcal{S}} := \sup_{\substack{\testop\in\mathcal{S}\\\testop\neq0}} \frac{\norm{\testop}_1}{\norm{\testop}_2} <\infty$. Then, for every $\delta\in(0,1]$, with probability at least $1-\delta$, simultaneously for every $\iteridx\geq\iteridx_0$ and $\testop\in\mathcal{S}$,
\begin{equation}
    \abs{\pairing{\targetop-\learntop_\iteridx,\testop}}
    \leq
    \frac{
        \regfactor C_{\mathcal{S}}
        \norm{\targetop}_\opnorm
        \norm{\testop}_2
    }{
        \regfactor+\kappa\iteridx
    }
    +
    \frac{
        \sigma_\vnoise\norm{\testop}_2
    }{
        \sqrt{\regfactor+\kappa\iteridx}
    }
    \sqrt{
        2\log\frac{1}{\delta}
        +
        d\log\paren*{
            1+\frac{\iteridx L^2}{\regfactor}
        }
    }.
    \label{eq:linear-persistent-excitation-rate}
\end{equation}
Consequently, for fixed $\delta$ and $\testop\in\mathcal{S}$,
\begin{equation*}
    \abs{\pairing{\targetop-\learntop_\iteridx,\testop}}
    =
    O\paren*{
        \sqrt{\frac{\log\iteridx}{\iteridx}}
    }.
\end{equation*}
\end{corollary}

\begin{proof}
Since $\operatorname{Ran}(\obsop_i)\subseteq\mathcal{S}$ for every $i$, the subspace $\mathcal{S}$ is invariant under $\operator{G}_\iteridx$, and hence also under $\regsumop_\iteridx=\regfactor\idop+\operator{G}_\iteridx$. By \eqref{eq:persistent-excitation}, for every $\testop\in\mathcal{S}$ and $\iteridx\geq\iteridx_0$,
\begin{equation*}
    \norm{
        \regsumop_\iteridx^{-1/2}(\testop)
    }_2
    \leq
    \frac{
        \norm{\testop}_2
    }{
        \sqrt{\regfactor+\kappa\iteridx}
    }.
\end{equation*}
Moreover, $\regsumop_\iteridx^{-1}(\testop)\in\mathcal{S}$, so finite-dimensional equivalence of the trace and Hilbert-Schmidt norms on $\mathcal{S}$ gives
\begin{equation*}
    \norm{
        \regsumop_\iteridx^{-1}(\testop)
    }_1
    \leq
    C_{\mathcal{S}}
    \norm{
        \regsumop_\iteridx^{-1}(\testop)
    }_2
    \leq
    \frac{
        C_{\mathcal{S}}\norm{\testop}_2
    }{
        \regfactor+\kappa\iteridx
    }.
\end{equation*}

Next, since $\vncov_i\preceq\sigma_\vnoise^2\idop$,
\begin{equation*}
    \sigma_\vnoise^{-2}\sumcov_\iteridx
    \preceq
    \operator{G}_\iteridx.
\end{equation*}
Furthermore, $\operator{G}_\iteridx$ has rank at most $d$ and
\begin{equation*}
    \norm{\operator{G}_\iteridx}_\opnorm
    \leq
    \sum_{i=1}^\iteridx
    \norm{\obsop_i\obsop_i^\adj}_\opnorm
    \leq
    \iteridx L^2.
\end{equation*}
Hence, by monotonicity of the determinant over positive-semidefinite finite-rank operators,
\begin{equation*}
    \det\paren{
        \idop+
        \regfactor^{-1}
        \sigma_\vnoise^{-2}
        \sumcov_\iteridx
    }
    \leq
    \det\paren{
        \idop+
        \regfactor^{-1}
        \operator{G}_\iteridx
    }
    \leq
    \paren*{
        1+\frac{\iteridx L^2}{\regfactor}
    }^d.
\end{equation*}
Substituting these three bounds into \eqref{eq:linear-regression-bound} gives \eqref{eq:linear-persistent-excitation-rate}. For fixed $\delta$, $\testop$, and problem constants, the regularization term is $O(\iteridx^{-1})$, while the stochastic term is $O(\sqrt{\log\iteridx/\iteridx})$, proving the final claim.
\end{proof}

\subsubsection{Proof of \autoref{thr:nonlinear-regression}}

\begin{proof}
Let $\widetilde{\mdlop}_\iteridx$ be the minimizer of $\Loss_\iteridx$ over the whole space $\mdlspace$. Since $\mdlop_\star=\mdlop(\cdot,\parameter_\star)$ belongs to the parametric model class and $\widehat{\parameter}_\iteridx$ is a global solution of \eqref{eq:nonlinear-parametric-estimator}, we have that:
\begin{equation*}
    \forall\iteridx\in\N, \quad \Loss_\iteridx\bigl(\mdlop(\cdot,\widehat{\parameter}_\iteridx)\bigr)
    \leq
    \Loss_\iteridx(\mdlop_\star).
\end{equation*}
Applying \autoref{lem:loss-difference} first to $\mdlop(\cdot,\widehat{\parameter}_\iteridx)$ and then to $\mdlop_\star$ leads us to:
\begin{equation*}
    \begin{aligned}
        \frac{1}{2}
        \norm{
            \mdlop(\cdot,\widehat{\parameter}_\iteridx)
            -
            \widetilde{\mdlop}_\iteridx
        }_{\operator{H}_\iteridx}^2
        &\leq
        \Loss_\iteridx\bigl(
            \mdlop(\cdot,\widehat{\parameter}_\iteridx)
        \bigr)
        -
        \Loss_\iteridx(\widetilde{\mdlop}_\iteridx)\\
        &\leq
        \Loss_\iteridx(\mdlop_\star)
        -
        \Loss_\iteridx(\widetilde{\mdlop}_\iteridx)\\
        &\leq
        \frac{1}{2}
        \norm{
            \nabla\Loss_\iteridx(\mdlop_\star)
        }_{\operator{H}_\iteridx^{-1}}^2,
    \end{aligned}
\end{equation*}
while the same lemma also yields:
\begin{equation*}
    \frac{1}{2}
    \norm{
        \mdlop_\star-\widetilde{\mdlop}_\iteridx
    }_{\operator{H}_\iteridx}^2
    \leq
    \frac{1}{2}
    \norm{
        \nabla\Loss_\iteridx(\mdlop_\star)
    }_{\operator{H}_\iteridx^{-1}}^2.
\end{equation*}
The triangle inequality therefore implies that:
\begin{equation}
    \forall\iteridx\in\N, \quad \norm{
        \mdlop(\cdot,\widehat{\parameter}_\iteridx)-\mdlop_\star
    }_{\operator{H}_\iteridx}
    \leq
    2
    \norm{
        \nabla\Loss_\iteridx(\mdlop_\star)
    }_{\operator{H}_\iteridx^{-1}}.
    \label{eq:nonlinear-gradient-error}
\end{equation}

By the Hilbert-space chain rule and the definition of $\vnoise_\iteridx$,
\begin{equation*}
    \nabla\Loss_\iteridx(\mdlop_\star)
    =
    \sum_{i=1}^\iteridx
    \obsop_i\vnoise_i
    +
    \nabla R_\iteridx(\mdlop_\star).
\end{equation*}
For each $i\in\range{\iteridx}$, $\obsop_i\vnoise_i$ is conditionally $\obsop_i\vncov_i\obsop_i^\adj$-sub-Gaussian. This covariance proxy is trace class because $\obsop_i$ is Hilbert--Schmidt and $\vncov_i$ is bounded. Moreover, as $\vncov_i\preceq\sigma_\vnoise^2\idop$,
\begin{equation}
    \sumcov_\iteridx
    =
    \sum_{i=1}^\iteridx
    \obsop_i\vncov_i\obsop_i^\adj
    \preceq
    \sigma_\vnoise^2
    \sum_{i=1}^\iteridx
    \obsop_i\obsop_i^\adj.
    \label{eq:sumcov-ub}
\end{equation}
Since $\regfactor_\iteridx\geq\regfactor_0$, it follows that:
\begin{equation*}
    \frac{
        \sigma_\vnoise^2\regfactor_0
    }{\alpha}\idop
    +
    \sumcov_\iteridx
    \preceq
    \frac{\sigma_\vnoise^2}{\alpha}
    \left(
        \regfactor_\iteridx\idop
        +
        \alpha
        \sum_{i=1}^\iteridx
        \obsop_i\obsop_i^\adj
    \right)
    =
    \frac{\sigma_\vnoise^2}{\alpha}
    \operator{H}_\iteridx.
\end{equation*}
Inverting this operator inequality, we then have that:
\begin{equation*}
    \norm*{
        \sum_{i=1}^\iteridx\obsop_i\vnoise_i
    }_{\operator{H}_\iteridx^{-1}}
    \leq
    \frac{\sigma_\vnoise}{\sqrt{\alpha}}
    \norm*{
        \sum_{i=1}^\iteridx\obsop_i\vnoise_i
    }_{
        \left(
            \frac{
                \sigma_\vnoise^2\regfactor_0
            }{\alpha}\idop
            +
            \sumcov_\iteridx
        \right)^{-1}
    }.
\end{equation*}
Applying \autoref{thr:sn-sequence} with regularization operator $(\sigma_\vnoise^2\regfactor_0/\alpha)\idop$ shows that, with probability at least $1-\delta$, simultaneously for every $\iteridx\in\N$,
\begin{equation*}
    \norm*{
        \sum_{i=1}^\iteridx\obsop_i\vnoise_i
    }_{\operator{H}_\iteridx^{-1}}
    \leq
    \frac{\sigma_\vnoise}{\sqrt{\alpha}}
    \sqrt{
        2\log\paren*{
            \frac{
                \det\paren*{
                    \idop+
                    \frac{\alpha}{
                        \sigma_\vnoise^2\regfactor_0
                    }
                    \sumcov_\iteridx
                }^{1/2}
            }{\delta}
        }
    },
\end{equation*}
whose determinant can be further simplified by \eqref{eq:sumcov-ub}.
Finally, the triangle inequality applied to the gradient decomposition gives
\begin{equation*}
    \norm*{
        \nabla\Loss_\iteridx(\mdlop_\star)
    }_{\operator{H}_\iteridx^{-1}}
    \leq
    \norm*{
        \nabla R_\iteridx(\mdlop_\star)
    }_{\operator{H}_\iteridx^{-1}}
    +
    \norm*{
        \sum_{i=1}^\iteridx\obsop_i\vnoise_i
    }_{\operator{H}_\iteridx^{-1}}.
\end{equation*}
Substitution into \eqref{eq:nonlinear-gradient-error} proves \eqref{eq:nonlinear-regression-error}.
\end{proof}



\end{document}